\documentclass[10pt,conference]{IEEEtran}
\IEEEoverridecommandlockouts
\usepackage{cite}
\usepackage{amsmath,amssymb,amsfonts,amsthm}
\usepackage{algorithm}      
\usepackage{algorithmicx} 
\usepackage{algpseudocode}
\usepackage{graphicx}
\usepackage{textcomp}
\usepackage{xcolor}
\usepackage{colortbl}
\usepackage{wrapfig}
\usepackage{graphicx}
\usepackage{subcaption} 
\usepackage{booktabs}
\numberwithin{equation}{section}

\numberwithin{equation}{section}
\newtheorem{definition}{Definition}[section]
\newtheorem{assumption}{Assumption}[section]
\newtheorem{lemma}{Lemma}[section]
\newtheorem{theorem}{Theorem}[section]

\def\BibTeX{{\rm B\kern-.05em{\sc i\kern-.025em b}\kern-.08em
    T\kern-.1667em\lower.7ex\hbox{E}\kern-.125emX}}
\begin{document}

\title{ADF-LoRA: Alternating Low-Rank Aggregation for Decentralized Federated Fine-Tuning}

\author{
Xiaoyu Wang, Xiaotian Li, Zhixiang Zhou, Chen Li, and Yong Liu\\
Department of Electrical and Computer Engineering, New York University, Brooklyn, NY, USA \\
Email: wang.xiaoyu@nyu.edu, xl3399@nyu.edu, zz4819@nyu.edu, chen.lee@nyu.edu, yongliu@nyu.edu
}

\maketitle

\begin{abstract}
This paper revisits alternating low-rank updates for federated fine-tuning and examines their behavior in decentralized federated learning (DFL). While alternating the LoRA matrices has been shown to stabilize aggregation in centralized FL, extending this mechanism to decentralized, peer-to-peer communication introduces new challenges due to phase-state mismatch and block-wise divergence across clients. We introduce ADF-LoRA, which synchronizes the update of only one low-rank matrix per round and mixes both matrices to maintain more consistent parameter states under decentralized propagation. This design preserves the cross-term suppression effect of alternating updates while improving stability in serverless topologies. We provide a convergence analysis under standard smoothness assumptions and evaluate ADF-LoRA on multiple GLUE tasks. Experiments show that ADF-LoRA achieves faster and smoother convergence and delivers the highest average accuracy across tasks, outperforming existing LoRA variants in decentralized FL by a consistent margin.
\end{abstract}


\section{Introduction}

Large Language Models (LLMs) such as GPT~\cite{achiam2023gpt},
LLaMA~\cite{touvron2023llama}, GLM~\cite{zengglm}, and DeepSeek~\cite{liu2024deepseek}
have achieved state-of-the-art performance on diverse language tasks. Adapting
these models in a privacy-preserving manner motivates federated learning (FL),
where data remain distributed across clients.
Parameter-efficient fine-tuning (PEFT) methods, especially LoRA~\cite{hulora},
are attractive for FL due to their small memory and communication footprints.
LoRA introduces trainable low-rank factors $A$ and $B$, but aggregating these
matrices directly in FL creates a bilinear inconsistency: averaging $A$ and $B$
independently produces cross-client interactions $B_iA_j$ ($i\neq j$) that distort
the global update and degrade convergence. 
Several recent studies have attempted to mitigate this challenge. FLoRA~\cite{wangflora} eliminates cross terms via block-diagonal stacking, at the cost of parameter growth proportional to the number of clients. FlexLoRA~\cite{bai2024federated} reconstructs unified updates using SVD, but suffers from scalability issues due to costly server-side factorization. FFA-LoRA~\cite{sunimproving} fixes one LoRA matrix across all clients to prevent bilinear interference, though at the cost of reduced expressiveness.

A recent line of work shows that \emph{alternating} the optimization of $A$ and
$B$ stabilizes LoRA aggregation in \emph{centralized} FL (CFL) by ensuring that
only one block is updated and aggregated per round~\cite{chen2025robust}.
However, this mechanism relies on strict synchronization: all clients must hold
identical parameters and follow identical phases. In decentralized FL (DFL),
where clients mix parameters with only their neighbors, information propagates
gradually and client states drift within each phase~\cite{wang2025decentralized}. As a result, the key
assumptions enabling cross-term suppression in centralized alternating LoRA no longer hold in the DFL setting. 

This paper revisits alternating LoRA through the lens of DFL and asks: {\it how can alternating low-rank updates be made stable and effective under peer-to-peer communication?} 
We show that decentralized mixing interacts with block-coordinate updates
in a nontrivial way, leading to phase-state mismatch and drift between
LoRA directions.

To address these issues, we introduce \texttt{ADF-LoRA}, a decentralized alternating LoRA
framework that incorporates (i) interval-based directional switching and (ii)
joint mixing of both LoRA blocks to maintain cross-client alignment. This yields a more stable alternating process and keeps decentralized LoRA training effective even on challenging NLP tasks.

\subsection{Contributions}
Our main contributions are summarized as follows:
\begin{itemize}
    \item We analyze why alternating LoRA, while stable in centralized FL, becomes unreliable in decentralized settings, showing that peer-to-peer mixing introduces phase-state mismatch and block-wise drift.
    
    \item We introduce \texttt{ADF-LoRA}, an alternating scheme for DFL that synchronizes the update of a single low-rank block per round and applies joint mixing of both blocks to maintain more consistent parameter states across clients.
    
    \item We study how the switching interval $T$ influences stability and show that moderate intervals (e.g., $T=5$) provide the best balance between block-wise coordination and optimization flexibility.
    
    \item Extensive experiments on four GLUE tasks demonstrate that \texttt{ADF-LoRA} achieves faster and smoother convergence and attains the best average accuracy across tasks, with clear improvements on more challenging datasets such as QQP and MNLI.
\end{itemize}

\section{Related Work}
\subsection{Federated Learning}
Federated learning (FL) enables training across distributed clients without sharing raw data~\cite{mcmahan2017communication}. Decentralized FL~\cite{lian2017can, lalitha2019peer} removes the central server, causing client models to drift and complicating aggregation. We study this setting under LoRA-based fine-tuning, where the bilinear parameterization introduces unique inconsistencies in both CFL and DFL.

\subsection{Federated Learning}
Federated learning (FL) enables collaborative training across distributed
clients without sharing raw data~\cite{mcmahan2017communication}. Key
challenges include data heterogeneity~\cite{zhao2018federated}, partial client
participation~\cite{li2020federated}, and high communication
cost~\cite{kairouz2021advances}. To reduce communication overhead, prior work
explores model compression~\cite{sattler2019robust}, sparse
updates~\cite{aji2017sparse}, personalized models~\cite{fallah2020personalized},
and adaptive aggregation schemes.

Decentralized FL~\cite{lian2017can, lalitha2019peer, wang2025decentralized} and
asynchronous FL~\cite{xie2019asynchronous} relax or remove the central server,
making client states evolve at different rates and complicating convergence
analysis. Our work examines these issues specifically for LoRA-based low-rank
adaptation, where the bilinear parameterization introduces unique aggregation
inconsistencies under both CFL and DFL.

\subsection{LoRA in Federated Learning}
LoRA~\cite{hulora} is widely adopted in FL due to its low communication cost and
parameter efficiency. Several variants extend LoRA to heterogeneous and
personalized settings, including SLoRA~\cite{babakniya2023slora},
FedSA-LoRA~\cite{guo2025selective}, and
pFedLoRA~\cite{yi2024pfedloramodelheterogeneouspersonalizedfederated}. These
methods address client diversity but do not resolve the bilinear inconsistency
inherent to aggregating independently trained LoRA factors.

A complementary line of work tackles this issue directly. FLoRA~\cite{wangflora}
uses block-diagonal concatenation to avoid cross terms but enlarges model size;
FlexLoRA~\cite{bai2024federated} reconstructs global updates via SVD; and
FFA-LoRA~\cite{sunimproving} freezes one LoRA matrix to suppress cross-client
interference. While effective, these methods trade off scalability, flexibility,
or representational capacity.

Most relevant to this work, RoLoRA~\cite{chen2025robust} shows that alternating
updates of $A$ and $B$ improve stability in \emph{centralized} FL. However, its
analysis assumes synchronous rounds and server-enforced parameter consistency,
and evaluates only fixed odd--even schedules under centralized aggregation.
Whether alternating LoRA remains stable in decentralized peer-to-peer settings—
where clients mix stale parameters and cannot maintain a shared frozen block—
has not been investigated, and naive extensions behave poorly in practice.

Beyond federated settings, AltLoRA~\cite{yu2025altlora} applies alternating
projections to improve gradient approximation in centralized single-machine
fine-tuning. This literature focuses on optimization quality rather than
multi-client aggregation and is orthogonal to our study of alternating LoRA
under communication-constrained FL.

\section{Motivation and Analysis}

LoRA models a weight update using a low-rank factorization $\Delta W_i = B_iA_i$.
However, federated learning aggregates the \emph{parameters} $(A_i,B_i)$ rather
than the product. Averaging the factors across clients yields
\[
A^{\mathrm{agg}}=\sum_i w_iA_i,\qquad
B^{\mathrm{agg}}=\sum_i w_iB_i,
\]
and reconstructing the update gives
\[
B^{\mathrm{agg}}A^{\mathrm{agg}}
= \sum_i w_i^2 B_iA_i
\;+\;
\underbrace{\sum_{i\neq j} w_iw_j\, B_iA_j}_{\text{cross-client interference}}.
\]
The first term corresponds to the true average update, while the second mixes
$B_i$ from one client with $A_j$ from another---an update direction that no
client ever computed. These cross terms are the primary source of instability
when LoRA is trained on heterogeneous data.

\paragraph{Why alternating works in CFL?}
Recent work shows that in \emph{centralized} FL (CFL), alternating updates
eliminate these cross terms~\cite{chen2025robust}. In a B-phase all clients
share the same $A$, so $B_iA$ averages cleanly; in an A-phase they share the
same $B$, so $BA_i$ averages cleanly. Server-enforced synchronization guarantees that the ``frozen'' block
is identical across clients, making each aggregated update cross-term-free and
restoring the stability of two-block coordinate descent.

\paragraph{Why this breaks in DFL?}
In DFL, clients exchange parameters only with neighbors, so
their models are no longer identical within a phase. Even if all clients agree
to be in an A-phase, their local $A_i$ differ due to incomplete propagation
through the network. Consequently, a B-phase no longer eliminates cross terms:
\[
\begin{aligned}
&\text{client $i$ holds } A_i \text{ and receives } A_j \neq A_i\text{ from neighbors} \\
&\Longrightarrow
\text{B-phase can no longer eliminate the } B_i A_j \text{ cross terms}.
\end{aligned}
\]
This introduces two DFL-specific effects:

\begin{itemize}
    \item \textbf{Phase-state mismatch:} clients enter the same logical phase
    with different versions of the ``frozen'' block, violating the assumption of
    a shared $A$ or $B$ that centralized alternating relies on.
    \item \textbf{Block-wise drift:} the frozen block gradually diverges across
    clients, so the cross-client interactions that alternating was designed to
    suppress reappear during mixing.
\end{itemize}

These effects fundamentally alter the behavior of alternating LoRA in DFL.
\textit{Our goal is therefore to understand: 1) how decentralized propagation affects
alternating LoRA? 2) whether simple coordination mechanisms can restore its
stability without requiring a central server?}

\section{Methodology}
The above analysis shows that decentralized communication disrupts the
synchronization required for alternating LoRA to remain stable, mainly due to
phase-state mismatch and block-wise drift. To study and mitigate these effects,
we first formalize the decentralized FL communication model, and then introduce
our joint-mixing alternating LoRA procedure.
\subsection{Decentralized FL Communication Model}
We adopt a standard decentralized FL (DFL) setting where clients maintain local
parameters and exchange them with neighbors through a time-varying mixing
matrix $W_t$:
\[
x_i^{t+1}=\sum_{j=1}^N (W_t)_{ij}\, x_j^t.
\]
The matrices $\{W_t\}$ are \emph{doubly-stochastic} and satisfy standard connectivity and
spectral-gap assumptions used in decentralized optimization. All clients perform local updates every round. Communication is serverless:
each round consists of (i) local computation and (ii) optional peer-to-peer
mixing when communication is triggered. This contrasts with centralized FL
(CFL), where a parameter server enforces identical global states at each round.



\subsection{Alternating LoRA in FL}

LoRA represents a low-rank update as $\Delta W = BA$, where $A$ and $B$
denote the down- and up-projection matrices. Alternating LoRA updates only
one block per phase:
\[
A^{t+1}_i = A^t_i - \eta\nabla_A\mathcal{L}_i(A^t_i,B^t_i),\qquad  
B^{t+1}_i = B^t_i,
\]
during the A-phase, and symmetrically for the B-phase. This block splitting 
removes LoRA cross terms and stabilizes optimization.

\paragraph{CFL: Centralized Alternating LoRA (RoLoRA).}
In centralized FL, all clients follow identical A/B phases and the server conducts model aggregation through FedAvg:
\[
A^{t+1} = \frac1N\sum_{i=1}^N A^{t+1}_i,\qquad
B^{t+1} = \frac1N\sum_{i=1}^N B^{t+1}_i.
\]
Since only the active block is aggregated in each phase, the centralized
procedure behaves as standard two-block coordinate descent.

\paragraph{DFL: Decentralized Alternating LoRA Baseline.}
In DFL, aggregation is replaced by peer-to-peer mixing, but only the \emph{active}
block is synchronized. During an A-phase:
\[
A^{t+1}_i = \sum_j (W_t)_{ij} A^t_j, \qquad
B^{t+1}_i = B^t_i,
\]
and during a B-phase:
\[
B^{t+1}_i = \sum_j (W_t)_{ij} B^t_j, \qquad
A^{t+1}_i = A^t_i.
\]
Because the frozen block is not synchronized across clients during its inactive
phase, clients gradually accumulate
inconsistent versions of it, leading to \emph{phase-state mismatch} and
block-wise drift, the main failure mode of alternating LoRA under DFL.
\subsection{ADF-LoRA: Joint-Mixing Alternating LoRA}
To mitigate drift in decentralized alternating LoRA, we introduce
\texttt{ADF-LoRA}, which synchronizes both LoRA blocks in every round,
regardless of which block is being updated. After local updates, each client
performs joint mixing:
\[
A^{t+1}_i = \sum_j (W_t)_{ij} A^t_j,\qquad
B^{t+1}_i = \sum_j (W_t)_{ij} B^t_j.
\]
Joint mixing keeps the frozen block aligned across the network, preventing
phase-state mismatch and stabilizing alternating LoRA under decentralized
communication.

Algorithm~\ref{alg:ADFLoRA} summarizes the full procedure.

\begin{algorithm}[t]
\caption{ADF-LoRA: Joint-Mixing Alternating LoRA in Decentralized FL}
\label{alg:ADFLoRA}
\begin{algorithmic}[1]
\Require Initial $(A_i^0,B_i^0)$; total rounds $K$; phase length $T$; mixing matrices $W_t$.
\For{$t=0,\ldots,K-1$}
    \If{$\lfloor t/T\rfloor$ is even}  \Comment{B-phase}
        \State Each client updates $B^t_i$ (e.g., AdamW); $A^t_i$ frozen
    \Else  \Comment{A-phase}
        \State Each client updates $A^t_i$; (e.g., AdamW); $B^t_i$ frozen
    \EndIf
    \State \textbf{Joint mixing for all clients:}
    \[
     A^{t+1}_i \gets \sum_j (W_t)_{ij}A^t_j,\;\; B^{t+1}_i \gets \sum_j (W_t)_{ij}B^t_j\]
\EndFor
\end{algorithmic}
\end{algorithm}

\subsection{Theoretical Remarks}

We begin by stating the regularity assumptions used in our decentralized
convergence analysis.

\begin{assumption}[Block-wise Regularity \cite{chen2025robust}]
For any fixed $A$, the mapping $B\mapsto\mathcal{L}(A,B)$ is $L$-smooth and
$\mu$-weakly convex. Symmetrically, for any fixed $B$, the mapping
$A\mapsto\mathcal{L}(A,B)$ is also $L$-smooth and $\mu$-weakly convex.
\end{assumption}

This assumption ensures that each alternating update behaves as a smooth (and
possibly nonconvex) coordinate-wise gradient step. Under this condition, the
classical descent lemma applies block-wise, and in the centralized setting
alternating LoRA reduces to a two-block coordinate descent method with the
standard $O(1/(TK))$ stationarity rate.

\begin{assumption}[Mixing Matrix]
The communication matrix $W$ is symmetric, doubly-stochastic, and satisfies
\[
\rho=\left\| W-\tfrac1N\mathbf{1}\mathbf{1}^\top \right\| < 1 .
\]
\end{assumption}

The constraint $\rho<1$ corresponds to the underlying communication graph
having a positive spectral gap. Consequently, repeated multiplication by $W$
contracts disagreements across clients at a geometric rate. Writing $U^t$ for
the stacked local parameters, the consensus deviation
$\|U^t - \bar U^t \mathbf{1}^\top\|$ measures the drift introduced by
decentralized communication.

Under these assumptions, our decentralized alternating LoRA method satisfies
the following convergence guarantee.

\begin{theorem}[Decentralized Convergence]
After $K$ full periods ($2KT$ block updates), the averaged iterate satisfies
\begin{align*}
\min_{0\le t<2KT}
\|\nabla\mathcal{L}(\bar A^t,\bar B^t)\|^2
&\;\le\;
\frac{2\bigl(\mathcal{L}(A^0,B^0)-\mathcal{L}^*\bigr)}{\eta\,2KT}
\\[0.3em]
&\quad+\;
\frac{2C}{2KT}
\sum_{t=0}^{2KT-1}
\bigl\|U^t-\bar U^t\mathbf1^\top\bigr\|^2 .
\end{align*}
Moreover, since
\[
\|U^t-\bar U^t\mathbf1^\top\|
\;\le\;
\rho^t\,\|U^0-\bar U^0\mathbf1^\top\|,
\]
the consensus error term vanishes geometrically as $t\to\infty$, and thus
$\nabla\mathcal{L}(\bar A^t,\bar B^t)\to 0$.
\end{theorem}

The theorem cleanly decomposes decentralized convergence into two effects:  
(i)~the standard $O(1/(TK))$ descent dynamics obtained in the centralized
setting, and (ii)~an additional consensus error term induced purely by graph
connectivity. When communication is centralized
($W=\tfrac1N\mathbf1\mathbf1^\top$ and hence $\rho=0$), the consensus term
vanishes identically, recovering the classical two-block coordinate descent
rate as a special case. Full derivations are provided in Appendix~\ref{appendix}.

\section{Experiments}

\subsection{Experimental Setup}
We evaluate the proposed method on four representative GLUE tasks (SST-2, QNLI, QQP, MNLI) under both centralized (CFL) and decentralized (DFL) federated learning.

\subsubsection{Model Configuration.}
We use \textbf{RoBERTa-Large} (335M) with LoRA applied to the \textbf{Q}/\textbf{V} projections ($r=8$, $\alpha=16$, dropout~0.1).  
The classification head is frozen.

\subsubsection{Federated Learning Settings.}
We consider:
\begin{itemize}
    \item \textbf{CFL}: A server aggregates LoRA updates each round.
    \item \textbf{DFL}: Clients communicate in a peer-to-peer topology. Each client encounters another client with probability~0.1 per round, and encountered pairs perform symmetric LoRA averaging.
\end{itemize}
Data partitions are identical across CFL and DFL.  
For binary tasks, 10 clients follow $3\times[0.9,0.1]$, $3\times[0.1,0.9]$, $4\times[0.5,0.5]$;  
for MNLI, $4\times[0.9,0.05,0.05]$, $3\times[0.05,0.9,0.05]$, $3\times[0.05,0.05,0.9]$.

\subsubsection{Baselines.}
We compare:
\begin{itemize}
    \item \texttt{LoRA}: Naive LoRA fine-tuning.
    \item \texttt{FFA-LoRA}~\cite{sunimproving}: A frequency-filtered LoRA baseline.
    \item \texttt{RoLoRA}~\cite{chen2025robust}: Naive CFL alternating extends to DFL by aggregating only trainable parameters.
    \item \texttt{ADF-LoRA-$T$}: our method with $T\in\{1,5,10,20\}$.
\end{itemize}

\subsubsection{Training Details.}
Each FL round uses \textbf{20 local steps}, for \textbf{150 rounds} total.  
We use AdamW (HuggingFace defaults), sequence length 128, batch size 32, and search learning rate over  
$\{5\times10^{-4}, 10^{-3}, 2\times10^{-3}, 5\times10^{-3}\}$.  
All experiments run on NVIDIA A100/V100/RTX8000 clusters.
\subsubsection{Evaluation Protocol.}
In CFL, a global model is available; thus, evaluation is performed directly on the aggregated global model for each random seed.

In DFL, no global model exists. For each seed, we evaluate all 10 client models and compute the \textbf{mean accuracy across clients}.  
We then report the \textbf{mean and standard deviation across random seeds} as the final performance.

\subsection{Centralized FL: Sanity Check}
To verify correctness before decentralized evaluation, we benchmark all methods under CFL.
Figure~\ref{fig:cfl_sanity_plot} shows accuracy on four GLUE tasks with standard deviation across seeds.

The results indicate:
(1) \texttt{RoLoRA} matches the accuracy of full \texttt{LoRA},  
(2) \texttt{RoLoRA} consistently improves over \texttt{FFA-LoRA},  
suggesting that alternating updates do not degrade performance under centralized aggregation.

\begin{figure}[t]
    \centering
    \includegraphics[width=0.95\linewidth]{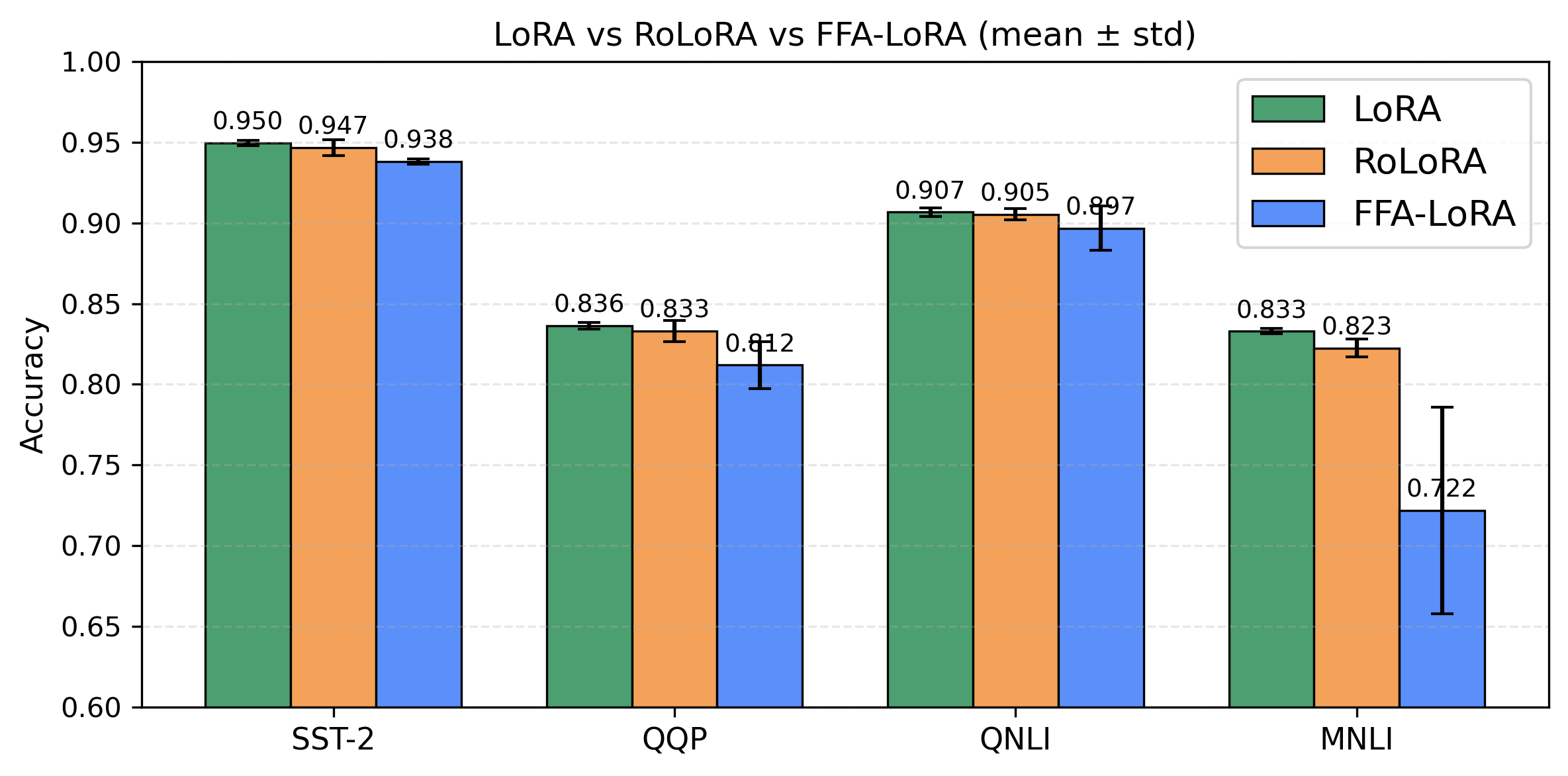}
    \caption{Accuracy and standard deviation across seeds for LoRA variants on four GLUE tasks under centralized FL.}
    \label{fig:cfl_sanity_plot}
\end{figure}

\subsection{Decentralized FL Results}

Evaluating convergence every round in DFL is computationally expensive for QQP and MNLI.  
Following common FL practice, we plot convergence curves only for QNLI and report final accuracy for all tasks.

\begin{figure}[t]
    \centering
    \includegraphics[width=0.95\linewidth]{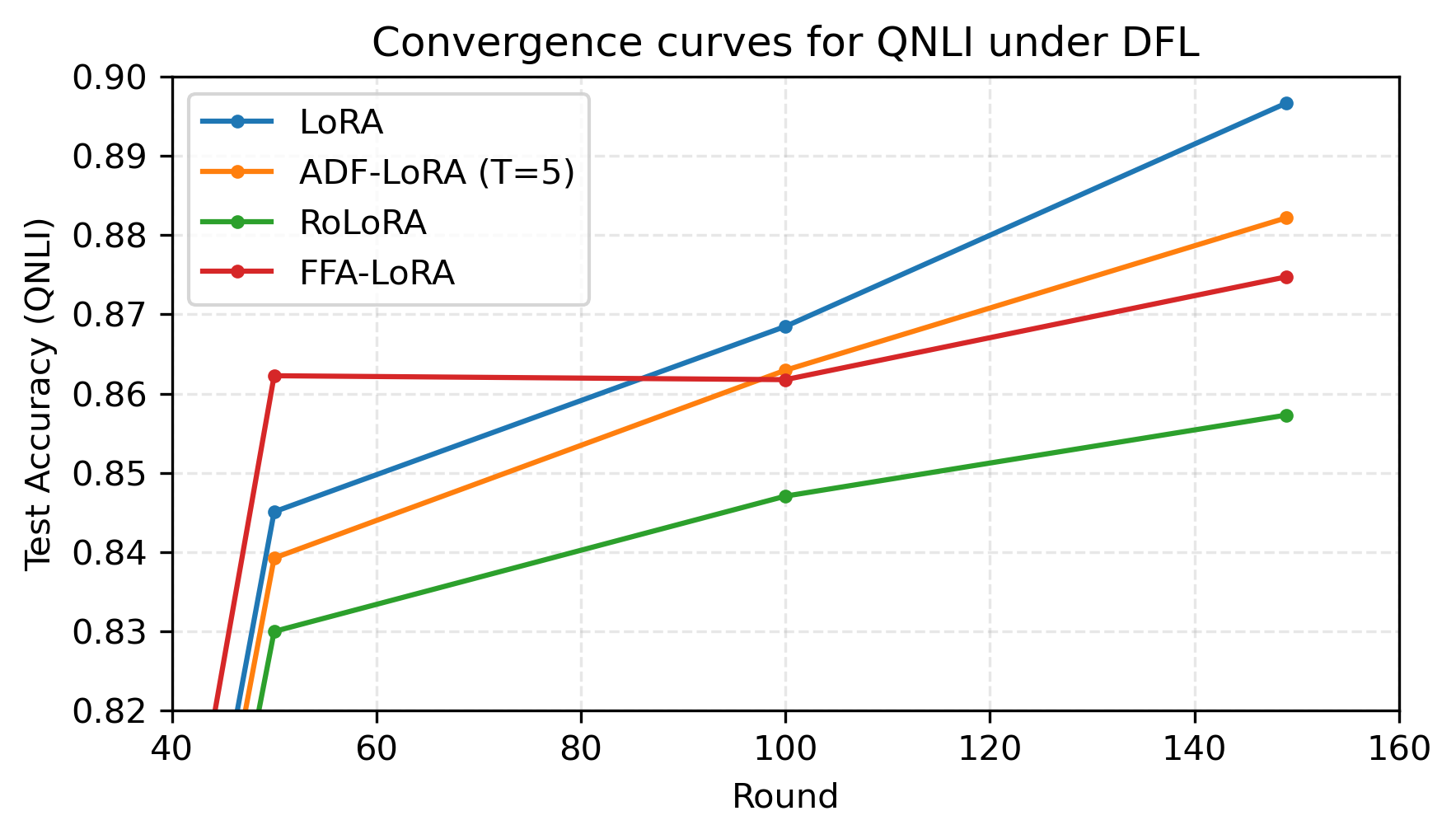}
    \caption{Convergence curves for QNLI under DFL.}
    \label{fig:qnli_curve}
\end{figure}
\subsubsection{Convergence Curves (QNLI)}
Figures~\ref{fig:qnli_curve} show the mean accuracy across 10 clients.  
\texttt{ADF-LORA} exhibit faster and more stable convergence compared with \texttt{RoLoRA} and \texttt{FFA-LoRA}.  
The alternating update scheme in \texttt{ADF-LORA} also suppresses the oscillation
commonly observed in decentralized training.
\begin{table*}[t]
\centering
\caption{Final accuracy ($\text{mean} \pm \text{std}$) after 150 DFL rounds. Best results in \textbf{bold}, second-best are \underline{underlined}.}
\label{tab:final_accuracy}
\begin{tabular}{lccccc}
\toprule
Method & SST-2 & QNLI & QQP & MNLI & Avg \\
\midrule
LoRA & 
$\mathbf{0.9482} \pm 0.0018$ &
$\mathbf{0.8970} \pm 0.0004$ &
$\underline{0.8077} \pm 0.0097$ &
$\underline{0.7304} \pm 0.0267$ &
\underline{$0.8458$} \\

FFA-LoRA &
$0.9329 \pm 0.0017$ &
$0.8758 \pm 0.0015$ &
$0.7926 \pm 0.0143$ &
$0.7058 \pm 0.0055$ &
$0.8268$ \\

RoLoRA &
$0.9354 \pm 0.0051$ &
$0.8685 \pm 0.0159$ &
$0.7915 \pm 0.0074$ &
{$0.7184 \pm 0.0025$} &
$0.8284$ \\

\rowcolor{gray!12}
ADF-LoRA (ours, $T=5$) &
$\underline{0.9422} \pm 0.0041$ &
$\underline{0.8826} \pm 0.0005$ &
$\mathbf{0.8129} \pm 0.0022$ &
$\mathbf{0.7624} \pm 0.0137$ &
$\mathbf{0.8505}$ \\
\bottomrule
\end{tabular}
\end{table*}
\subsubsection{Final Accuracy on Four GLUE Tasks}
Table~\ref{tab:final_accuracy} reports the mean~$\pm$~std accuracy across 10 clients.  
With an interval of $T=5$, \texttt{ADF-LoRA} attains the highest average accuracy and the strongest gains on QQP and MNLI, while remaining competitive on SST-2 and QNLI.

\texttt{ADF-LoRA} combines the strengths of alternating LoRA updates with a more effective decentralized mixing strategy.  
Like RoLoRA, it removes the cross-term noise introduced by naive LoRA aggregation, but it additionally mixes both the updated and frozen blocks at every peer interaction.  
This design eliminates the \emph{phase-state mismatch} between clients updating different blocks and prevents the frozen block from drifting or becoming stale, thereby reducing block-wise inconsistency under sparse communication.  
By simultaneously suppressing cross-term noise and mitigating mismatch-induced drift, \texttt{ADF-LoRA} achieves consistently better performance than RoLoRA and even matches or surpasses full LoRA on the more complex QQP and MNLI tasks.

\begin{table}[t]
\centering
\caption{Ablation on interval $T$: final accuracy (mean only). Best results are in \textbf{bold}.}
\label{tab:ablation_T}
\begin{tabular}{lccccc}
\toprule
$T$ & SST-2 & QNLI & QQP & MNLI & Avg \\
\midrule
1  & 0.9393 & 0.8779 & 0.8056 & 0.7512 & 0.8435 \\
\textbf{5}  & \textbf{0.9422} & \textbf{0.8826} & \textbf{0.8129} & \textbf{0.7624} & \textbf{0.8505} \\
10 & 0.9362 & 0.8785 & 0.8032 & 0.7395 & 0.8393 \\
20 & 0.9348 & 0.8777 & 0.8124 & 0.7454 & 0.8421 \\
\bottomrule
\end{tabular}
\end{table}

\subsection{Ablation Study: Impact of Interval $T$}

\subsubsection{Convergence Behavior}
Figures~\ref{fig:ablation_T_curve} compare $T=1, 5, 10, 20$ on QNLI.  
Smaller $T$ (e.g., $1$) produces slower improvement, while larger $T$ (e.g., $20$) introduces noticeable fluctuations.  
Moderate intervals ($T=5$ or $T=10$) achieve the most stable convergence.

\subsubsection{Final Accuracy}
Table~\ref{tab:ablation_T} shows that the alternating interval $T$ has a clear impact on performance.  
A moderate interval achieves the best results, with $T=5$ giving the highest average accuracy across tasks.  
Both very frequent switching ($T=1$) and overly delayed switching ($T=20$) lead to suboptimal performance, indicating that an appropriate switching frequency is essential for balancing block-wise stability and cross-block coupling.  
Determining how to choose the optimal interval in a task- or system-aware manner is an interesting direction for future work.
\begin{figure}[t]
    \centering
     \includegraphics[width=0.95\linewidth]{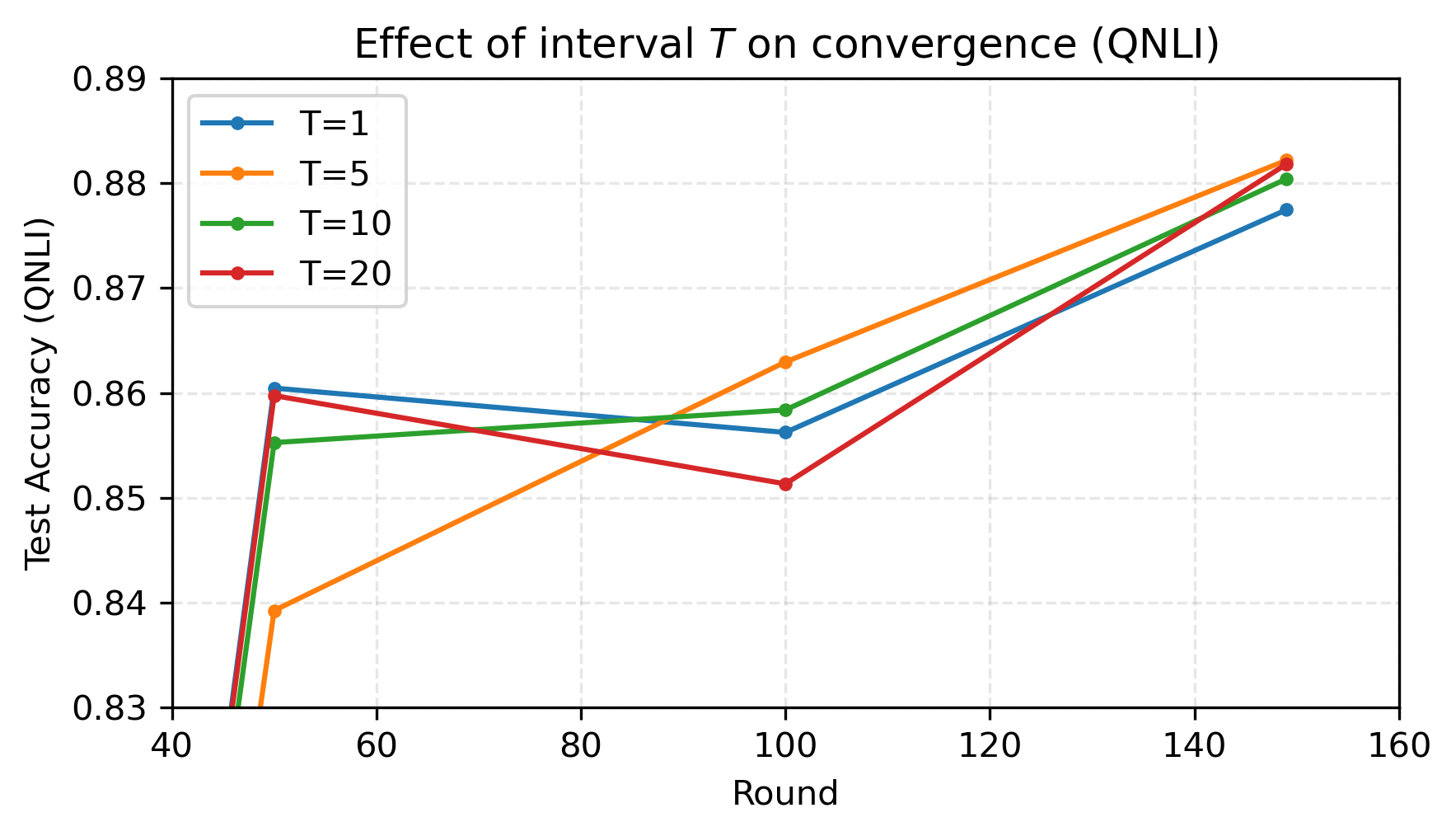}
    \caption{Effect of interval $T$ on convergence on QNLI.}
    \label{fig:ablation_T_curve}
\end{figure}

\subsection{Summary of Findings}

Our experiments yield the following key observations:

\begin{itemize}

    \item \textbf{CFL sanity check.}  
    RoLoRA matches full LoRA and outperforms FFA-LoRA, confirming that alternating updates behave correctly under centralized aggregation.

    \item \textbf{DFL advantages.}  
    In decentralized training, \texttt{ADF-LoRA} converges faster and more smoothly than RoLoRA and FFA-LoRA, especially on QNLI.  
    Its improved mixing strategy helps maintain consistent LoRA directions across clients.

    \item \textbf{Gains on difficult tasks.}  
    \texttt{ADF-LoRA} achieves the best results on QQP and MNLI and the highest overall average accuracy, showing particular advantages on tasks with higher optimization complexity.

    \item \textbf{Optimal interval $T$.}  
    Moderate switching intervals (notably $T=5$) perform best, while both very frequent ($T=1$) and infrequent ($T=20$) switching degrade performance, suggesting the need for an appropriate interval.
\end{itemize}

Overall, \texttt{ADF-LoRA} refines alternating LoRA updates by retaining RoLoRA’s cross-term noise reduction while improving consistency under decentralized mixing, resulting in more stable and accurate decentralized fine-tuning.

\section{Conclusion}
We introduced \texttt{ADF-LoRA}, a refined alternating update framework for decentralized federated fine-tuning. 
\texttt{ADF-LoRA} preserves RoLoRA’s cross-term noise reduction while improving the consistency of alternating updates under sparse peer-to-peer communication through coordinated mixing of both LoRA blocks. 
This design enables stable decentralized optimization and achieves strong empirical performance across a range of GLUE tasks.

\textbf{Future Work.}
Promising directions include adaptive switching policies, scaling \texttt{ADF-LoRA} to larger language models, and studying its behavior under other  decentralized topologies, asynchronous communication protocols, and large-scale peer-to-peer environments.

\bibliographystyle{IEEEtran}
\bibliography{reference}

\onecolumn
\appendices
\section{Convergence of Alternating LoRA SGD in Decentralized FL}\label{appendix}
We optimize
\[
\mathcal{L}(A,B)
= \frac{1}{N}\sum_{i=1}^N \mathcal{L}_i(A,B),
\]
with LoRA blocks $A\in\mathbb R^{r\times d}$, $B\in\mathbb R^{d\times r}$. Agents
communicate via a fixed, symmetric doubly-stochastic matrix
$W\in\mathbb R^{N\times N}$ with spectral gap
$\rho=\|W-\tfrac1N\mathbf1\mathbf1^\top\|<1$.

\begin{definition}[$L$-smoothness {\cite{xie2019asynchronous}}]
A differentiable $f$ is $L$-smooth if
\[
f(y)\le f(x)+\langle\nabla f(x),\,y-x\rangle+\frac{L}{2}\|y-x\|^2.
\]
\end{definition}

\begin{definition}[$\mu$-weak convexity {\cite{xie2019asynchronous}}]
$f$ is $\mu$-weakly convex if $f(x)+\tfrac\mu2\|x\|^2$ is convex.
\end{definition}

\begin{assumption}[Block-wise Regularity]
For any fixed $A$, $B\mapsto\mathcal{L}(A,B)$ is $L$-smooth and $\mu$-weakly
convex; and vice versa.
\end{assumption}

\begin{assumption}[Mixing Matrix]
$W$ is symmetric, doubly-stochastic, and
$\rho=\|W-\tfrac1N\mathbf1\mathbf1^\top\|<1$.
\end{assumption}

\begin{assumption}[SGD Step Size]
Agents use SGD with a constant step size $\eta \le 1/L$.\footnote{%
While our experiments use AdamW~\cite{loshchilov2018decoupled},
our theoretical analysis considers vanilla SGD for clarity. 
Recent work shows that adaptive gradient methods such as Adam
exhibit convergence behavior similar to SGD under standard
smoothness and bounded-variance assumptions
(e.g., \cite{j.2018on}). 
Thus, the theoretical insights for alternating LoRA are expected to
extend to Adam-type optimizers in practice.}
\end{assumption}

Let each period be $2T$ steps: for $t=0,\dots,2KT-1$ define
\[
u_t =
\begin{cases}
B, & t\bmod(2T)<T,\\
A, & t\bmod(2T)\ge T,
\end{cases}
\]
and let each agent $i$ perform a local update on the active block:
\[
u_i^{t+\tfrac12}=u_i^t-\eta\,\nabla_{u_t}\mathcal{L}_i(A^t_i,B^t_i).
\]
After each local update, \emph{both} LoRA blocks are mixed:
\[
A_i^{t+1}=\sum_{j=1}^N W_{ij} A_j^{t+\tfrac12}, \qquad
B_i^{t+1}=\sum_{j=1}^N W_{ij} B_j^{t+\tfrac12}.
\]
Denote averages $\bar A^t=\tfrac1N\sum_i A_i^t$ and
$\bar B^t=\tfrac1N\sum_i B_i^t$.

\begin{lemma}[Consensus Error Decay Under Joint Mixing]
Under the joint mixing update above, both blocks satisfy the same contraction:
\[
\|U^{t+1}-\bar U^{t+1}\mathbf1^\top\|
\;\le\;
\rho\,\|U^{t}-\bar U^{t}\mathbf1^\top\|,
\qquad U\in\{A,B\},
\]
where $\bar U^t=\tfrac1N\sum_i U_i^t$.
\end{lemma}

\begin{lemma}[Descent with Consensus Error]
Under Assumptions~1–4, there exists a constant $C>0$ such that for each step $t$,
\[
\mathcal{L}(\bar A^{t+1},\bar B^{t+1})
\le
\mathcal{L}(\bar A^t,\bar B^t)
-\frac{\eta}{2}\bigl\|\nabla_{u_t}\mathcal{L}(\bar A^t,\bar B^t)\bigr\|^2
+ C\,\bigl\|U^{t}-\bar U^{t}\mathbf1^\top\bigr\|^2,
\]
where $u_t\in\{A,B\}$ is the block updated at step $t$ and $U^t$ is the
corresponding stacked local variable.
\end{lemma}

\begin{proof}
For clarity we first treat the deterministic case where each agent uses the
exact local gradient $\nabla_{u_t}\mathcal{L}_i(A_i^t,B_i^t)$; the extension to
stochastic gradients with bounded variance follows standard arguments in
decentralized SGD.

Let $u_t$ be the active block at step $t$ and fix the other block. Denote
\[
g_i^t \;=\; \nabla_{u_t}\mathcal{L}_i(A_i^t,B_i^t),\qquad
g^t \;=\; \frac1N\sum_{i=1}^N g_i^t.
\]
The local gradient step is
\[
u_i^{t+\frac12}=u_i^t-\eta g_i^t.
\]
Since for the active block we have
$u_i^{t+1} = \sum_j W_{ij} u_j^{t+\frac12}$ and $W$ is doubly stochastic,
the average after the mixing step is
\[
\bar u^{t+1}
=\frac1N\sum_{i=1}^N u_i^{t+1}
=\frac1N\sum_{i=1}^N u_i^{t+\frac12}
=\bar u^t-\eta g^t.
\]
Moreover the average of the non–updated block remains unchanged, so
\[
(\bar A^{t+1},\bar B^{t+1})
=(\bar A^{t},\bar B^{t})-\eta\,e_t,
\]
where $e_t$ is the vector that has $g^t$ in the coordinates of $u_t$ and zeros
in the other block.
where $e_t$ is the vector that has $g^t$ in the coordinates of $u_t$ and zeros in the other block.

By block-wise $L$–smoothness (Assumption~1),
\[
\begin{aligned}
\mathcal{L}(\bar A^{t+1},\bar B^{t+1})
&\le
\mathcal{L}(\bar A^t,\bar B^t)
+\left\langle \nabla_{u_t}\mathcal{L}(\bar A^t,\bar B^t),\,\bar u^{t+1}-\bar u^t\right\rangle
+\frac{L}{2}\|\bar u^{t+1}-\bar u^t\|^2 \\
&=
\mathcal{L}(\bar A^t,\bar B^t)
-\eta\left\langle \nabla_{u_t}\mathcal{L}(\bar A^t,\bar B^t),\,g^t\right\rangle
+\frac{L\eta^2}{2}\|g^t\|^2 .
\end{aligned}
\]
We now decompose $g^t$ as
\[
g^t
=\nabla_{u_t}\mathcal{L}(\bar A^t,\bar B^t)+\Delta^t,
\quad
\Delta^t
:=g^t-\nabla_{u_t}\mathcal{L}(\bar A^t,\bar B^t).
\]
Substituting and expanding,
\[
\begin{aligned}
\mathcal{L}(\bar A^{t+1},\bar B^{t+1})
&\le
\mathcal{L}(\bar A^t,\bar B^t)
-\eta\bigl\|\nabla_{u_t}\mathcal{L}(\bar A^t,\bar B^t)\bigr\|^2
-\eta\left\langle \nabla_{u_t}\mathcal{L}(\bar A^t,\bar B^t),\,\Delta^t\right\rangle\\
&\quad
+\frac{L\eta^2}{2}\Bigl(\bigl\|\nabla_{u_t}\mathcal{L}(\bar A^t,\bar B^t)\bigr\|^2
+2\bigl\langle\nabla_{u_t}\mathcal{L}(\bar A^t,\bar B^t),\Delta^t\bigr\rangle
+\|\Delta^t\|^2\Bigr).
\end{aligned}
\]
Collecting the coefficients of $\|\nabla_{u_t}\mathcal{L}\|^2$ and using $\eta\le 1/L$, we obtain
\[
-\eta+\frac{L\eta^2}{2}\le -\frac{\eta}{2}.
\]
Using Cauchy–Schwarz and Young’s inequality with a small constant, we can absorb the mixed terms
$\langle\nabla_{u_t}\mathcal{L},\Delta^t\rangle$ into $\|\nabla_{u_t}\mathcal{L}\|^2$ and $\|\Delta^t\|^2$,
which yields
\[
\mathcal{L}(\bar A^{t+1},\bar B^{t+1})
\le
\mathcal{L}(\bar A^t,\bar B^t)
-\frac{\eta}{2}\bigl\|\nabla_{u_t}\mathcal{L}(\bar A^t,\bar B^t)\bigr\|^2
+ C_1\|\Delta^t\|^2
\]
for some $C_1>0$.

Finally, standard arguments in decentralized optimization (e.g., \cite{xie2019asynchronous,lian2017can}) show that the gradient disagreement $\|\Delta^t\|$ can be bounded by the consensus error,
\[
\|\Delta^t\|^2
\le C_2\,\bigl\|U^{t}-\bar U^{t}\mathbf1^\top\bigr\|^2
\]
for some $C_2>0$ depending on the Lipschitz constants of $\nabla\mathcal{L}_i$.
Setting $C=C_1C_2$ gives the desired inequality.
\end{proof}

\begin{theorem}[Decentralized Convergence]
After $K$ full periods ($2KT$ steps),
\[
\min_{0\le t<2KT}\|\nabla\mathcal{L}(\bar A^t,\bar B^t)\|^2
\;\le\;
\frac{2\bigl(\mathcal{L}(A^0,B^0)-\mathcal{L}^*\bigr)}{\eta\,2KT}
\;+\;
\frac{2C}{2KT}\sum_{t=0}^{2KT-1}\bigl\|U^t-\bar U^t\mathbf1^\top\bigr\|^2.
\]
Moreover, since $\|U^t-\bar U^t\mathbf1^\top\|\le\rho^t\|U^0-\bar U^0\mathbf1^\top\|$, the consensus error term vanishes as $t\to\infty$, and thus $\nabla\mathcal{L}(\bar A^t,\bar B^t)\to0$.
\end{theorem}

\begin{proof}
Summing the descent inequality in Lemma~2 over $t=0,\dots,2KT-1$ gives
\[
\begin{aligned}
\mathcal{L}(\bar A^{2KT},\bar B^{2KT})
&\le
\mathcal{L}(\bar A^0,\bar B^0)
-\frac{\eta}{2}\sum_{t=0}^{2KT-1}
\bigl\|\nabla_{u_t}\mathcal{L}(\bar A^t,\bar B^t)\bigr\|^2
+ C\sum_{t=0}^{2KT-1}\bigl\|U^t-\bar U^t\mathbf1^\top\bigr\|^2\\[2mm]
&\le
\mathcal{L}(A^0,B^0)
-\frac{\eta}{2}\sum_{t=0}^{2KT-1}
\bigl\|\nabla_{u_t}\mathcal{L}(\bar A^t,\bar B^t)\bigr\|^2
+ C\sum_{t=0}^{2KT-1}\bigl\|U^t-\bar U^t\mathbf1^\top\bigr\|^2,
\end{aligned}
\]
where we used that $\mathcal{L}(\bar A^0,\bar B^0)=\mathcal{L}(A^0,B^0)$ and dropped a nonnegative term $\mathcal{L}^*-\mathcal{L}(\bar A^{2KT},\bar B^{2KT})$ by lower-bounding $\mathcal{L}$ with $\mathcal{L}^*$.

Rearranging yields
\[
\frac{\eta}{2}\sum_{t=0}^{2KT-1}
\bigl\|\nabla_{u_t}\mathcal{L}(\bar A^t,\bar B^t)\bigr\|^2
\le
\mathcal{L}(A^0,B^0)-\mathcal{L}^*
+ C\sum_{t=0}^{2KT-1}\bigl\|U^t-\bar U^t\mathbf1^\top\bigr\|^2.
\]
Dividing by $2KT$ and using that
\[
\min_{0\le t<2KT}\bigl\|\nabla_{u_t}\mathcal{L}(\bar A^t,\bar B^t)\bigr\|^2
\le
\frac{1}{2KT}
\sum_{t=0}^{2KT-1}
\bigl\|\nabla_{u_t}\mathcal{L}(\bar A^t,\bar B^t)\bigr\|^2,
\]
we obtain
\[
\min_{0\le t<2KT}\bigl\|\nabla_{u_t}\mathcal{L}(\bar A^t,\bar B^t)\bigr\|^2
\le
\frac{2\bigl(\mathcal{L}(A^0,B^0)-\mathcal{L}^*\bigr)}{\eta\,2KT}
+\frac{2C}{2KT}\sum_{t=0}^{2KT-1}\bigl\|U^t-\bar U^t\mathbf1^\top\bigr\|^2.
\]
Since at each step $t$ either $u_t=A$ or $u_t=B$, and over each period all blocks are updated, vanishing partial gradients imply that
\[
\|\nabla \mathcal{L}(\bar A^t,\bar B^t)\|^2
=\|\nabla_A \mathcal{L}(\bar A^t,\bar B^t)\|^2
+\|\nabla_B \mathcal{L}(\bar A^t,\bar B^t)\|^2
\longrightarrow 0,
\]
so the averaged iterates converge to a first-order stationary point.

It remains to bound the consensus error term. Applying Lemma~1 recursively yields
\[
\bigl\|U^{t}-\bar U^{t}\mathbf1^\top\bigr\|
\le
\rho^{t}\bigl\|U^{0}-\bar U^{0}\mathbf1^\top\bigr\|,
\]
hence
\[
\sum_{t=0}^{2KT-1}\bigl\|U^{t}-\bar U^{t}\mathbf1^\top\bigr\|^2
\le
\bigl\|U^{0}-\bar U^{0}\mathbf1^\top\bigr\|^2
\sum_{t=0}^{\infty}\rho^{2t}
=
\frac{\bigl\|U^{0}-\bar U^{0}\mathbf1^\top\bigr\|^2}{1-\rho^2}
=O\!\left(\frac{1}{1-\rho}\right).
\]
Dividing by $2KT$ gives the stated bound and shows that the consensus error contribution vanishes as $K\to\infty$, which implies $\nabla\mathcal{L}(\bar A^t,\bar B^t)\to 0$.
\end{proof}

\subsection*{Centralized FL as a Special Case}

For completeness, we record the simpler convergence result for centralized
federated learning (CFL), which corresponds to the case where all clients are
perfectly synchronized and the server aggregates exact averages at each step.

In CFL, all agents share identical parameters $(A^t,B^t)$ and the server
computes the averaged gradient
\[
g^t \;=\; \frac1N\sum_{i=1}^N \nabla_{u_t}\mathcal{L}_i(A^t,B^t)
\;=\; \nabla_{u_t}\mathcal{L}(A^t,B^t),
\]
where $u_t\in\{A,B\}$ denotes the active block at step $t$. The alternating
update is then
\[
u^{t+1} = u^t - \eta g^t,
\]
while the other block remains fixed.

\begin{theorem}[Convergence in Centralized FL]
Suppose Assumptions~1, 3, and 4 hold. Let each period consist of $2T$ steps as
in the decentralized case, and run alternating updates for $2KT$ steps. Then
\[
\min_{0\le t<2KT}
\bigl\|\nabla\mathcal{L}(A^t,B^t)\bigr\|^2
\;\le\;
\frac{2\bigl(\mathcal{L}(A^0,B^0)-\mathcal{L}^*\bigr)}{\eta\,2KT},
\]
that is, alternating LoRA in CFL converges to a first-order stationary point at
the standard $O(1/TK)$ rate.
\end{theorem}

\begin{proof}
The proof follows the same structure as Lemma~2 and Theorem~1, but without the
consensus error term. By block-wise $L$-smoothness and the centralized update
$u^{t+1} = u^t - \eta \nabla_{u_t}\mathcal{L}(A^t,B^t)$, we obtain
\[
\mathcal{L}(A^{t+1},B^{t+1})
\le
\mathcal{L}(A^t,B^t)
-\eta\bigl\|\nabla_{u_t}\mathcal{L}(A^t,B^t)\bigr\|^2
+\frac{L\eta^2}{2}
\bigl\|\nabla_{u_t}\mathcal{L}(A^t,B^t)\bigr\|^2.
\]
Using $\eta\le 1/L$ gives
\[
\mathcal{L}(A^{t+1},B^{t+1})
\le
\mathcal{L}(A^t,B^t)
-\frac{\eta}{2}\bigl\|\nabla_{u_t}\mathcal{L}(A^t,B^t)\bigr\|^2.
\]
Summing over $t=0,\dots,2KT-1$ and lower-bounding by $\mathcal{L}^*$ yields
\[
\frac{\eta}{2}\sum_{t=0}^{2KT-1}
\bigl\|\nabla_{u_t}\mathcal{L}(A^t,B^t)\bigr\|^2
\le
\mathcal{L}(A^0,B^0)-\mathcal{L}^*.
\]
Dividing by $2KT$ and using
\[
\min_{0\le t<2KT}\bigl\|\nabla_{u_t}\mathcal{L}(A^t,B^t)\bigr\|^2
\le
\frac1{2KT}
\sum_{t=0}^{2KT-1}
\bigl\|\nabla_{u_t}\mathcal{L}(A^t,B^t)\bigr\|^2
\]
gives the claimed bound. As in the decentralized case, each block is updated
within every period, so vanishing block-wise gradients imply
$\|\nabla\mathcal{L}(A^t,B^t)\|\to 0$.
\end{proof}

\end{document}